\documentclass{article}

\usepackage{ijcai17}
\usepackage{times}
\usepackage{amsmath,amssymb,amsthm}
\usepackage{color}
\usepackage{graphicx}
\usepackage{algorithm}
\usepackage{algpseudocode}
\usepackage{relsize}

\newcommand{\NBA}{\mbox{\tt{NBS}}}

\newcommand{\Cstar}{\mbox{$C^*$}}

\newcommand{\CurrentBest}{\mbox{$C$}}

\newcommand{\closedB}{\mbox{$Closed_B$}}
\newcommand{\closedF}{\mbox{$Closed_F$}}

\newcommand{\open}{\mbox{$Open$}}
\newcommand{\openF}{\mbox{$Open_F$}}

\newcommand{\fF}{\mbox{$f_F$}}

\newcommand{\cF}{\mbox{$c$}}
\newcommand{\hF}{\mbox{$h_F$}}

\newcommand{\openB}{\mbox{$Open_B$}}
\newcommand{\fB}{\mbox{$f_B$}}

\newcommand{\cB}{\mbox{$c$}}
\newcommand{\hB}{\mbox{$h_B$}}

\newcommand{\ready}{\mbox{\em ready}}
\newcommand{\wait}{\mbox{\em waiting}}
\newcommand{\readyD}{\mbox{\em ready}$_D$}
\newcommand{\readyF}{\mbox{\em ready}$_F$}
\newcommand{\readyB}{\mbox{\em ready}$_B$}
\newcommand{\waitD}{\mbox{\em waiting}$_D$}
\newcommand{\waitF}{\mbox{\em waiting}$_F$}
\newcommand{\waitB}{\mbox{\em waiting}$_B$}

\newtheorem{theorem}{Theorem}
\newtheorem{definition}{Definition}

\newtheorem{corollary}[theorem]{Corollary}

\newcommand{\comment}[1]{}

\definecolor{brown}{cmyk}{0,0.5,1,0.2}
\definecolor{myGreen}{rgb}{0,0.8,0.4}
\definecolor{myRed}{rgb}{1,0,0}
\definecolor{myBlue}{rgb}{0,0.6,1.0}
\definecolor{myPurple}{rgb}{0.7,0.0,1.0}

\newcommand{\RHA}[1]{{#1}}

\newcommand{\fadm}{\RHA{forward admissible}}
\newcommand{\fcons}{\RHA{forward consistent}}
\newcommand{\feh}{\RHA{front-to-end heuristic}}
\newcommand{\badm}{\RHA{backward admissible}}
\newcommand{\bcons}{\RHA{backward consistent}}

\newcommand{\DXBB}{\mbox{DXBB}}
\newcommand{\VC}{\mbox{VC}}

\title{Front-to-End Bidirectional Heuristic Search with Near-Optimal Node Expansions}

\author{
Jingwei Chen \\
Dept. of Comp. Sci.\\
University of Denver \\
USA \\
{\tt \small jingchen@cs.du.edu}
  \And Robert C. Holte \\
  Comp. Sci. Dept.\\
  University of Alberta \\
  Canada \\
  {\tt \small rholte@ualberta.ca} \\
  \And Sandra Zilles \\
    Comp. Sci. Dept.\\
    University of Regina \\
    Canada \\
     {\tt \small zilles@uregina.ca}
\And
Nathan R. Sturtevant  \\
Dept. of Comp. Sci.\\
University of Denver \\
USA \\
{\tt \small sturtevant@cs.du.edu } \\
}

\begin{document}

\maketitle

\begin{abstract}

It is well-known that any admissible unidirectional heuristic search algorithm must expand all states whose $f$-value is smaller than the optimal solution cost when using a consistent heuristic. Such states are called ``surely expanded'' (s.e.). A recent study characterized s.e.\ pairs of states for \emph{bidirectional}\/ search with consistent heuristics: if a pair of states is s.e.\ then at least one of the two states must be expanded.
This paper derives a lower bound, $\VC$, on the minimum number of expansions required to cover all s.e.\ pairs, and present a new admissible front-to-end bidirectional heuristic search algorithm, Near-Optimal Bidirectional Search (\NBA), that is guaranteed to do no more than $2\VC$ expansions. We further prove that no admissible front-to-end algorithm has a worst case better than $2\VC$. Experimental results show that \NBA\ competes with or outperforms existing bidirectional search algorithms, and often outperforms A* as well.

\end{abstract}

\section{Introduction}

One method of formally assessing the efficiency of a heuristic search algorithm is to establish upper and lower bounds on the number of nodes it expands on any given problem instance $I$. Such bounds can then be compared to a theoretical minimum that a competing heuristic search algorithm would have to expand on $I$.
In this context, one is interested in finding sufficient conditions for node expansion, i.e., conditions describing nodes that must provably be expanded by any competing algorithm.
In a unidirectional search an algorithm must expand every node whose $f$-value is less than the optimal solution cost; this condition establishes the optimality of A*~\cite{AstarOptimal}.

Sufficient conditions for node expansion have recently been developed for front-to-end bidirectional heuristic search~\cite{eckerle17sufficient}, but no existing front-to-end bidirectional search algorithm is provably optimal.
In this paper, we use these sufficient conditions
to derive a simple graph-theoretic characterization of nodes that must provably be expanded on a problem instance $I$ by any admissible front-to-end bidirectional search algorithm given a consistent heuristic. In particular, the set of nodes expanded must correspond to a vertex cover of a specific graph derived from $I$. We then adapt a known vertex cover algorithm~\cite{papadimitriou1982combinatorial}  into a new admissible front-to-end bidirectional search algorithm,  \NBA\ (Near-Optimal Bidirectional Search), and prove that \NBA\ never expands more than twice the number of nodes contained in a minimum vertex cover. Hence, the number of nodes expanded by \NBA\ is provably within a factor of two of optimal. 

We further establish that no admissible bidirectional front-to-end algorithm can be better than \NBA\ in the worst case. In that sense, we formally verify that \NBA\ is near-optimal in the general case and optimal in the worst case.
In an experimental study on a set of standard benchmark problems, \NBA\ either competes with or outperforms existing bidirectional search algorithms, and it often outperforms the unidirectional algorithm A*, especially when the heuristic is weak or the problem instance is hard.

\section{Related Work}

Bidirectional search has a long history, beginning with bidirectional brute force search~\cite{Nicholson66}, and proceeding to heuristic search algorithms such as BHPA~\cite{pohl71}. 
Other notable algorithms include BS*~\cite{KwaBS1989}, which avoids re-expanding states in both directions, and MM~\cite{holte2016mm}, which ensures that the search frontiers meet in the middle.
Along with these algorithms there have been explanations for the poor performance of bidirectional heuristic search, including that the frontiers miss~\cite{Nilsson82} or that the frontiers meet early, and a long time is spent proving the optimal solution~\cite{KaindlKainz97}. Recent work has refined this, showing that with strong heuristics the frontiers meet later~\cite{BarkerKorf15}.

\section{Terminology and Notation}

We use the same notation and terminology as~\cite{eckerle17sufficient}.
A state space $G$ is a finite directed graph whose vertices are states and whose edges are pairs of states.\footnote{If $G$ has multiple edges from state $u$ to state $v$, we ignore all but the cheapest of them.}
Each edge $(u,v)$ has a cost $c(u,v)\ge 0$. A forward path in $G$ is a finite sequence $U=(U_0,\ldots, U_n)$ of states in $G$ where $(U_i,U_{i+1})$ is an edge in $G$ for  $0\le i< n$.
We say that forward path $U$ contains edge $(u,v)$ if $U_i=u$ and $U_{i+1}=v$ for some $i$. 
Likewise, a backward path is a finite sequence $V=(V_0,\ldots, V_m)$ of states where $(V_i,V_{i+1})$ is a ``reverse'' edge, i.e. $(V_{i+1},V_i)$ is an edge in $G$ for  $0\le i< m$.
Backward path $V$ contains reverse edge $(u,v)$ if $V_i=u$ and $V_{i+1}=v$ for some $i$.
The reverse of path $V=(V_0,\ldots, V_m)$ is  $V^{-1}=(V_m,\ldots, V_0)$.
The cost of a reverse edge equals the cost of the corresponding original edge.
A path pair $(U,V)$ has a forward path ($U$) as its first component and a backward path ($V$) as its second component.

If $U$ is a path (forward or backward), $|U|$ is the number of edges in $U$,
$c(U)$ is the cost of $U$ (the sum of the costs of all the edges in $U$),
and
$U_i$ is the $i^{th}$ state in $U$ ($0 \le i \le |U|$).
$U_{|U|}$ is the last state in path $U$, which we also denote $end(U)$.
$\lambda_F = (start)$ and $\lambda_B = (goal)$ are the empty forward and backward paths from $start$ and $goal$, respectively.  Note that $end(\lambda_F)=start$ while $end(\lambda_B)=goal$. Both $\lambda_F$ and $\lambda_B$ have a cost of $0$.
Forward (backward, resp.) path $U$ is optimal if there is no cheaper forward (backward, resp.) path from $U_0$ to $end(U)$.
$d(u,v)$ is the distance from state $u$ to state $v$, i.e., the cost of the cheapest forward path from $u$ to $v$.
If there is no forward path from $u$ to $v$ then $d(u,v) = \infty$.
Given two states in $G$, $start$ and $goal$, a solution path is a forward path from $start$ to $goal$. $\Cstar=d(start,goal)$ is the cost of the cheapest solution path.

A heuristic maps an individual state in $G$ to a non-negative real number or to $\infty$.
Heuristic $h_F$ is \emph{\fadm}\ iff $h_F(u)\le d(u,goal)$ for all $u$ in $G$ and is \emph{\fcons}\ iff $h_F(u)\le d(u,u^\prime)+h_F(u^\prime)$ for all $u$ and $u^\prime$ in~$G$.
Heuristic $h_B$ is \emph{\badm}\ iff $h_B(v)\le d(start,v)$ for all $v$ in $G$ and is \emph{\bcons}\ iff $h_B(v)\le d(v^\prime,v)+h_B(v^\prime)$ for all $v$ and $v^\prime$ in~$G$.
For any forward path $U$ with $U_0=start$ define $f_F(U)=c(U)+h_F(end(U))$, and for any backward path $V$ with $V_0=goal$ define $f_B(V)=c(V)+h_B(end(V))$.

A problem instance is defined by specifying two {\feh}s, \hF\ and \hB, and a state space $G$
represented implicitly by a 5-tuple $(start,goal,c,expand_F,expand_B)$ consisting of
a start state ($start$), a goal state ($goal$), an edge cost function ($c$), a successor function ($expand_F$), and a predecessor function ($expand_B$).
The input to $expand_F$ is a forward path $U$. Its output is a sequence $(U^1,\ldots,U^n)$,
where each $U^k$ is a forward path consisting of $U$ followed by one additional state ($end(U^k)$) such that $(end(U),end(U^k))$ is an edge in $G$.
There is one $U^k$ for every state $s$ such that $(end(U),s)$ is an edge in $G$.
Likewise,
the input to $expand_B$ is a backward path $V$ and its output is a sequence $(V^1,\ldots,V^m)$,
where each $V^k$ is a backward path consisting of $V$ followed by one additional state ($end(V^k)$) such that $(end(V^k),end(V))$ is an edge in $G$.
There is one $V^k$ for every state $s$ such that $(s,end(V))$ is an edge in $G$.

Although the expand functions operate on paths, it is sometimes convenient to talk about states being expanded. We say state $u$ has been expanded if one of the expand functions has been applied to a path $U$ for which $end(U)=u$. Finally, we say that a state pair $(u,v)$ has been expanded if either $u$ has been expanded in the forward direction or $v$ has been expanded in the backward direction (we do not require both).

A problem instance is solvable if there is a forward path in $G$ from $start$ to $goal$.
$I_{AD}$ is the set of solvable problem instances in which \hF\ is \fadm\ and \hB\ is \badm. $I_{CON}$ is the subset of $I_{AD}$ in which \hF\ is \fcons\ and \hB\ is \bcons. A search algorithm is {\em admissible} iff it is guaranteed to return an optimal solution for any problem instance in $I_{AD}$.

We only consider \DXBB~\cite{eckerle17sufficient}
algorithms. These are deterministic algorithms that proceed by expanding states that have previously been generated and have only black-box access to the expand, heuristic, and cost functions.

\section{Sufficient Conditions for Node Expansion}

This paper builds on recent theoretical work~\cite{eckerle17sufficient} defining sufficient conditions for state expansion for 
bidirectional \DXBB\ search algorithms. While A* with a consistent heuristic necessarily expands all states with $\fF(s) < \Cstar$, in bidirectional search there is no single state that is necessarily expanded, as the search can proceed forward or backwards, avoiding the need to expand any single state. However, given a path pair $(U, V)$, that meets the following conditions, one of the paths' end-states must necessarily be expanded.

\begin{theorem}\label{thm:sufficient_condition}~\cite{eckerle17sufficient}
	Let $I = (G,h_F,h_B)\in I_{CON}$ have an optimal solution cost of $\Cstar$.
	If $U$ is an optimal forward path and $V$ is an optimal backward path such that $U_0=start$, $V_0=goal$, and:
	\[
		\max\{\fF(U),\fB(V),\cF(U)+\cB(V)\} < \Cstar\,,
	\]
	then, in solving problem instance $I$, any admissible \DXBB\ bidirectional front-to-end search algorithm must expand the state pair ($end(U), end(V))$.
\end{theorem}

The $\cF(U)+\cB(V)$ condition was not used by BS*, is degenerate in A*, but is used by MM. When the heuristic is weak, this is an important condition for early termination.

We now use these conditions for state-pair expansion to define a bipartite graph.

\begin{definition}\label{def:lb}
	For path pair $(U, V)$ define
	\[
		lb(U, V) =\max\{\fF(U),\fB(V),\cF(U)+\cB(V)\} \,.
	\]
\end{definition}

When \hF\ is forward admissible and \hB\ is backward admissible, $lb(U, V)$ is a lower bound on the cost of a solution path of the form $UZV^{-1}$, where $Z$ is a forward path from $end(U)$ to $end(V)$.

\begin{definition}\label{def:must_expand_graph}
	The Must-Expand Graph $G_{\mathit{MX}}(I)$ of problem instance $I = (G,\hF,\hB) \in I_{CON}$ is an undirected, unweighted bipartite graph defined as follows.
	For each state $u \in G$, there are two vertices in $G_{\mathit{MX}}(I)$, the left vertex $u_F$ and right vertex $u_B$.
	For each pair of states  $u, v \in G$, there is an edge in $G_{\mathit{MX}}(I)$ between $u_F$ and $v_B$ if and only if there exist an optimal forward path $U$ with $U_0=start$ and $end(U)=u$ and an optimal backward path $V$ with $V_0=goal$ and $end(V)=v$  such that $lb(U,V) < \Cstar$.  Thus, there is an edge in $G_{\mathit{MX}}(I)$ between $u_F$ and $v_B$ if and only if Theorem~\ref{thm:sufficient_condition} requires the state pair $(u,v)$ to be expanded.
\end{definition}

We illustrate this in Figures~\ref{fig:statespace} and \ref{fig:must_expand_graph}. Figure~\ref{fig:statespace} shows a problem instance $I = (G,\hF,\hB) \in I_{CON}$. In this example $a$ is the start state, $f$ is the goal, and $\Cstar = 3$.
Figure~\ref{fig:must_expand_graph} shows $G_{\mathit{MX}}(I)$, where $d$ refers to the cost of the shortest path to each state and $f$ refers to the $f$-cost of that path.
By construction, the edges in $G_{\mathit{MX}}(I)$ exactly correspond to the state pairs that must be expanded according to Theorem~\ref{thm:sufficient_condition}, and therefore any vertex cover for $G_{\mathit{MX}}(I)$ will, by definition, represent a set of expansions that covers all the required state pairs.  For example, one possible vertex cover includes exactly the vertices in the left side with at least one edge--$\{a_F,c_F,d_F,e_F\}$.  This represents expanding all the required state pairs in the forward direction.  This requires four expansions and is not optimal because the required state pairs can be covered with just three expansions: $a$ and $c$ in the forward direction and $f$ in the backward direction. This corresponds to a minimum vertex cover of $G_{\mathit{MX}}(I): \{a_F,c_F,f_B\}$.

\begin{figure}[b!]
	\centering
	\includegraphics[width=6.0cm]{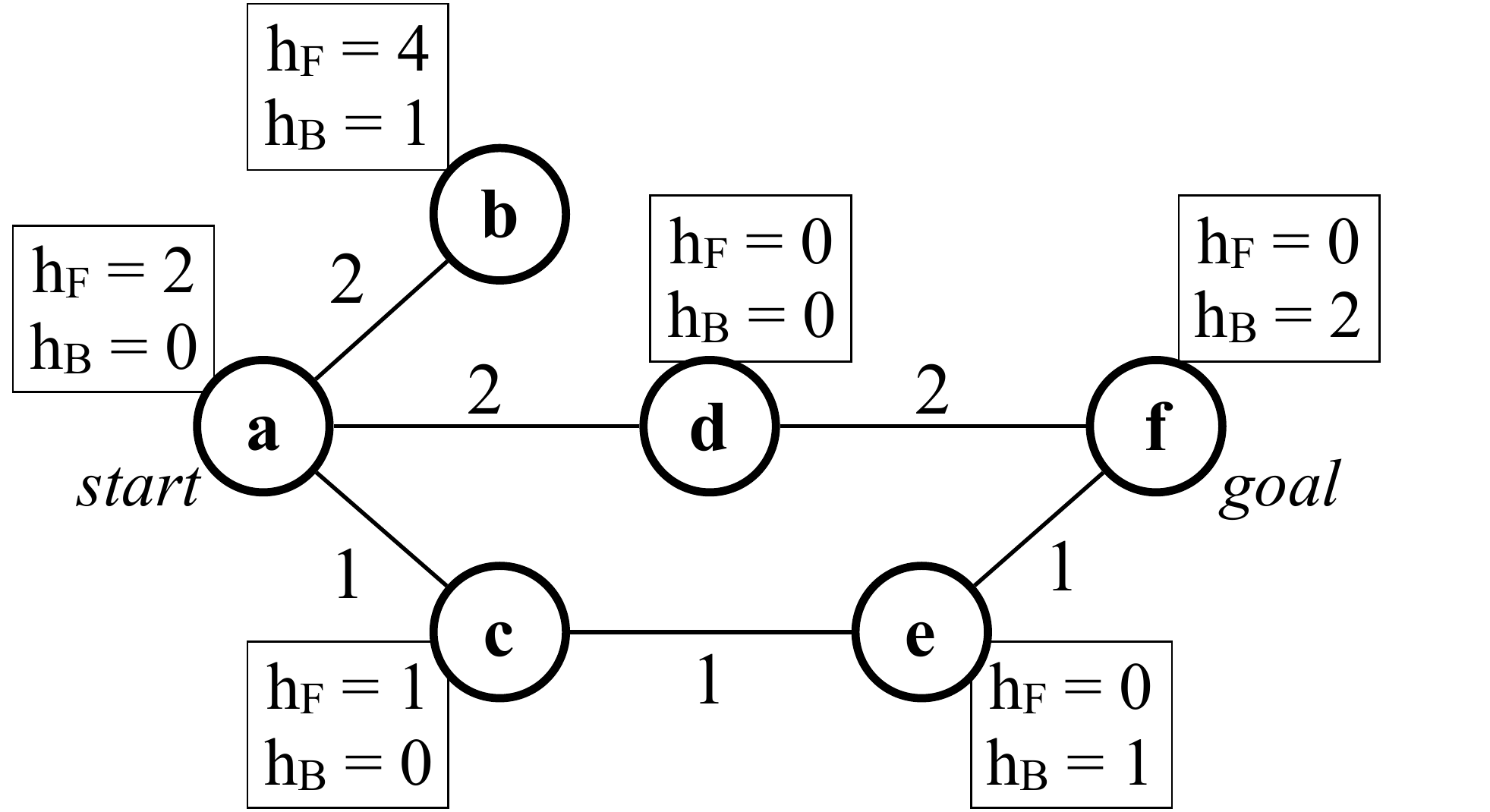}
	\caption{A sample problem instance.}
	\label{fig:statespace}
\end{figure}

\begin{figure}[b!]
	\centering
	\includegraphics[width=6cm]{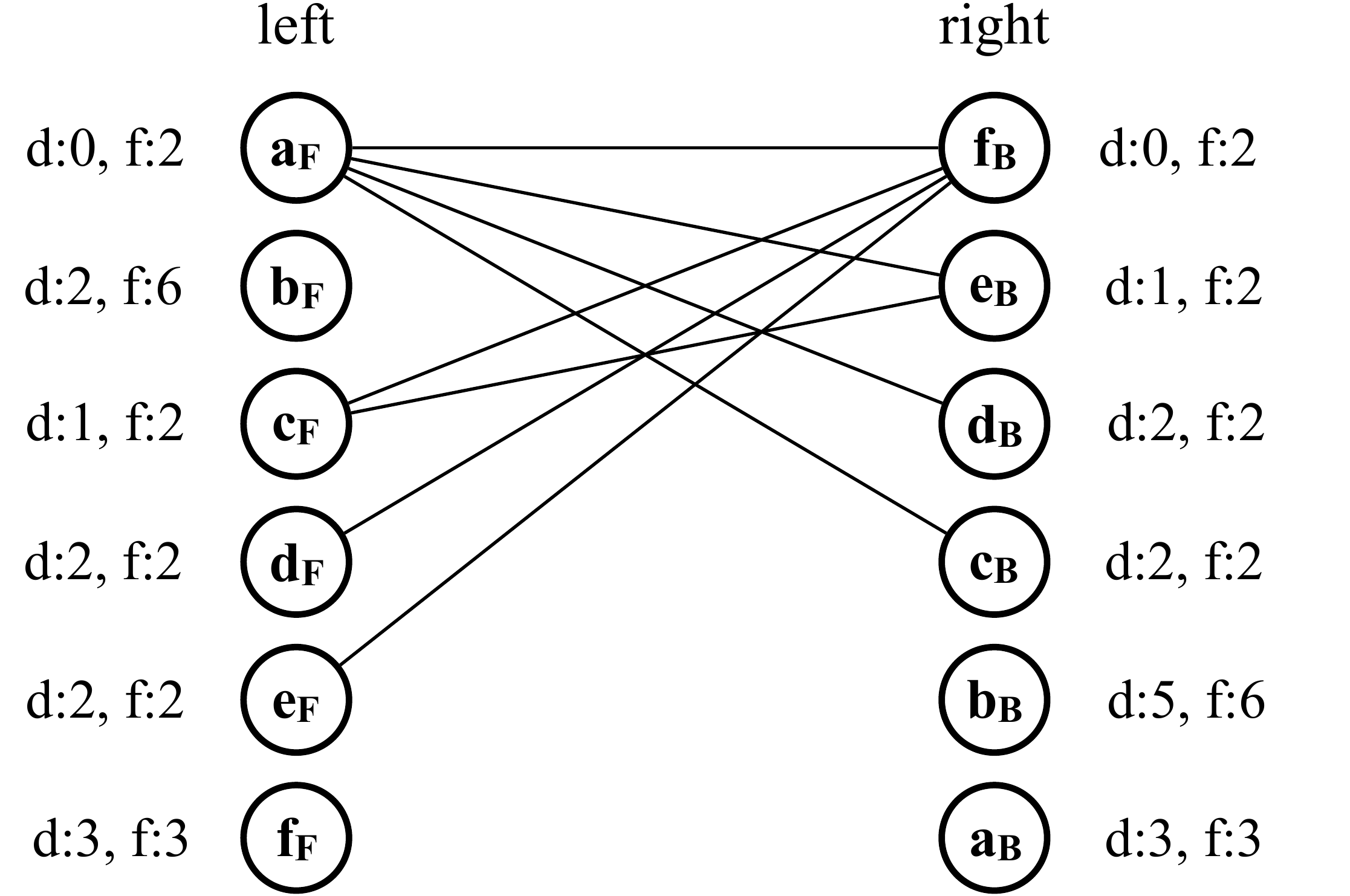}
	\caption{The Must-Expand Graph for Figure~\ref{fig:statespace}, where \Cstar=3.}
	\label{fig:must_expand_graph}
\end{figure}

\begin{theorem}\label{thm:goodnews}
	Let $I\in I_{CON}$. Let $A$ be an admissible \DXBB\ bidirectional front-to-end search algorithm, and $S_F$ (resp.\ $S_B$) be the set of states expanded  by $A$ on input $I$ in the forward (resp.\ backward) direction. Together, $S_F$ and $S_B$ correspond to a vertex cover for $G_{\mathit{MX}}(I)$. In particular, $|S_F|+|S_B|$ is lower-bounded by the size of the smallest vertex cover for $G_{\mathit{MX}}(I)$.
\end{theorem}

\noindent\emph{Proof.\/}
Let $(u_F,v_B)$ be an edge in $G_{\mathit{MX}}(I)$. Then Theorem~\ref{thm:sufficient_condition} requires the state pair $(u,v)$ to be expanded by $A$ on input $I$, i.e., $A$ must expand $u$ in the forward direction or $v$ in the backward direction. Thus the set of states expanded by $A$ on input $I$ corresponds to a vertex cover of $G_{\mathit{MX}}(I)$.
\qed

We will show below (Theorem~\ref{thm:worstcase}) that this lower bound cannot be attained by any admissible \DXBB\ bidirectional front-to-end search algorithm. However, we devise an {\em admissible} algorithm, \NBA, which efficiently finds a near-optimal vertex cover and thus is near-optimal in terms of necessary node expansions.

The claim of near-optimality is only with respect to the state pairs that must be expanded according to Theorem~\ref{thm:sufficient_condition}, it does not take into account state pairs of the form $(end(U),end(V))$ when $lb(U,V)=\Cstar$.  In principle, \NBA\ could expand many such state pairs while some other algorithm does not. We investigate this further in our experiments.

\section{\NBA: A Near-Optimal Front-to-End Bidirectional Search Algorithm}

While a vertex cover can be computed efficiently on a bipartite graph, in practice, building  $G_{\mathit{MX}}(I)$ is more expensive than solving $I$. Instead, we adapt a greedy algorithm for vertex cover~\cite{papadimitriou1982combinatorial} to achieve near-optimal performance. The greedy algorithm selects a pair of states that are not part of the vertex cover subset selected so far and are connected by an edge in the graph. It then adds both states to the vertex cover. We introduce a new algorithm, \NBA, which uses the same approach to achieve near-optimal node expansions while finding the shortest path.

The pseudocode for \NBA\ is shown in Algorithms~\ref{alg:nbs} and~\ref{alg:exp}.
\NBA\  considers all pairs for which  $lb$ is smallest (line~\ref{line:pairs}).
Among these pairs, it first chooses the pairs $(U,V)$ with smallest cost $c(U)$ (line~\ref{tie-breakingU}) and then, among those,
the ones with smallest cost $c(V)$ (line~\ref{tie-breakingV}). \NBA\ picks an arbitrary pair $(U,V)$ from the remaining candidates and expands both $U$ and $V$ (lines~\ref{line:expF}/\ref{line:expB}).
Breaking ties in this way is necessary to guarantee that \NBA\ never expands a suboptimal path when its heuristics are consistent; other tie-breaking rules can be used.
An efficient data structure for implementing this path pair selection is described in Section~\ref{sec:datastructure}.

\begin{algorithm}[b!]
	\caption{\NBA}
	\relsize{-1}
	\begin{algorithmic}[1]
		\State $\CurrentBest$ $\leftarrow$ $\infty$
		\State $\openF \leftarrow \{\lambda_F\}$; $\openB \leftarrow \{\lambda_B\}$
		\State $\closedF \leftarrow \emptyset$ ; $\closedB \leftarrow \emptyset$
		\While {$\openF \neq \emptyset$ and $\openB \neq \emptyset$}\label{lin:loopfirst}
		\State $Pairs \leftarrow \openF \times \openB$	
		\State $lbmin \leftarrow \min\{lb(X,Y)\mid (X,Y) \in Pairs\}$
		\If{$lbmin \geq \CurrentBest$} \label{line:stop}
		\Return $\CurrentBest$
		\EndIf
		\State $minset \leftarrow \{(X,Y) \in Pairs \mid lb(X,Y)=lbmin\}$ \label{line:pairs}
		\State $Uset \leftarrow \{X\mid \exists Y (X,Y) \in minset\}$
		\State $Umin \leftarrow \min\{c(X)\mid X \in Uset \}$
		\State Choose any $U\in Uset$ such that $c(U)=Umin$\label{tie-breakingU}
		\State $Vset \leftarrow \{Y \mid (U,Y) \in minset\}$
		\State $Vmin \leftarrow \min\{c(Y)\mid Y \in Vset \}$
		\State Choose any $V\in Vset$ such that $c(V)=Vmin$\label{tie-breakingV}
		\State Forward-Expand($U$) \label{line:expF}
		\State Backward-Expand($V$) \label{line:expB}
		\EndWhile \label{lin:looplast}
		\State \Return $\CurrentBest$
	\end{algorithmic} \label{alg:nbs}
\end{algorithm}


\begin{algorithm}[t!]
	\caption{\NBA: Forward-Expand($U$) }
	\relsize{-1}
	\begin{algorithmic}[1]
		\State Move $U$ from \openF\ to \closedF\label{lin:closeU}
		\For {{\bf each} $W \in expand_F(U)$ }\label{lin:generate}
		\If{$\exists Y \in \openB$ with $end(Y)=end(W)$}  
		\State $\CurrentBest = \min(\CurrentBest,c(W)+c(Y))$
		\EndIf
		\If{$\exists X \in \openF \cup \closedF$ with $end(X)=end(W)$}\label{lin:DD}
		\If{$c(X) \le c(W)$}\label{lin:comparecosts}
		\State Continue for loop // discard $W$\label{lin:discard2}
		\Else
		\State remove $X$ from $\openF / \closedF$\label{lin:removeX}
		\EndIf
		\State Add $W$ to \openF\label{lin:addW}
		\EndIf
		\EndFor
	\end{algorithmic} \label{alg:exp}
\end{algorithm}

The pseudocode for backwards expansion is not shown, as it is analogous to forward expansion.
We have proofs that, for all problem instances in $I_{AD}$, \NBA\ returns \Cstar.
These are not included here because of space limitations, and
because they are very similar to the corresponding proof for MM.

\section{Bounded Suboptimality in State Expansions}

Theorem~\ref{thm:sufficient_condition} identifies the set of state pairs that must be expanded by any admissible \DXBB\ front-to-end bidirectional algorithm. We refer to these as surely expanded (s.e.) path pairs.
The theorem does not stipulate which state in each s.e. pair must be expanded; an algorithm is free to make that choice in any manner.  Different choices can lead to vastly different numbers of expansions.
Given that $\VC$ is the size of a minimum vertex cover for $G_{\mathit{MX}}$, we have shown above that at least $\VC$ expansions are required.
In this section we prove that on the subset of consistent problem instances \NBA\ never expands more than $2\VC$ states to cover all the s.e. pairs, and that for every \DXBB\ front-to-end bidirectional algorithm $A$ there exists a problem instance in $I_{CON}$ on which $A$ expands at least $2\VC$ states to cover all the s.e. pairs.
That means that the suboptimality of \NBA\ is bounded by a factor of two, and that no competing algorithm can do better in the worst case.

\begin{theorem}\label{thm:factor2}
Let $I\in I_{CON}$, let $G_{\mathit{MX}}(I)$ be the Must-Expand Graph, and let $\VC(I)$ be the size of the smallest vertex cover of $G_{\mathit{MX}}$. Then \NBA\ does no more than $2\VC(I)$ state expansions on $G_{\mathit{MX}}(I)$ to cover its s.e. pairs.
\end{theorem}

\begin{proof}
If $(u,v)$ is a s.e. pair then $lb(U,V) < \Cstar$ for every optimal forward path $U$ from $start$ to $u$ and every optimal backward path from $goal$ to $v$. \NBA\ will select exactly one such $(U,V)$ pair for expansion and expand both $end(U)=u$ and $end(V)=v$.
A minimum vertex cover for $I$ might require only one of them to be expanded, so for each expansion required by a minimum vertex cover, \NBA\ might do two.
\end{proof}

Theorems~\ref{thm:goodnews} and \ref{thm:factor2} yield the following result.

\begin{corollary}\label{cor:bound}
Let $I\in I_{CON}$ and let $A$ be any admissible front-to-end DXBB bidirectional algorithm. Then \NBA\ makes no more than twice the number of state expansions on input $I$ than $A$ does in covering $I$'s s.e. pairs.
\end{corollary}

For any algorithm $A$, let us use the term \emph{worst-case expansion ratio of $A$}\/ to refer to the ratio $\max_{I \in I_{CON}}\frac{\#A(I)}{\#(I)}$, where $\#A(I)$ is the number of states $A$ expands in covering the s.e. pairs in instance $I$, and $\#(I)$ is the smallest number of states any
DXBB front-to-end bidirectional search algorithm expands in covering the s.e. pairs in $I$.
By definition, $\#(I) \ge VC(I)$, so we can
rephrase Corollary~\ref{cor:bound} as follows: 
\begin{equation}\label{eq:2x}
	\mbox{\NBA's worst-case expansion ratio is at most 2.}
\end{equation}

We now demonstrate that \NBA\ is optimal in the sense that no admissible DXBB front-to-end bidirectional search algorithm  has a worst-case expansion ratio smaller than 2.

\begin{theorem}\label{thm:worstcase}
	Let $A$ be any admissible DXBB front-to-end bidirectional search algorithm. Then there exists a problem instance $I$ and a DXBB front-to-end bidirectional search algorithm  $B$ such that $A$ expands at least twice as many states in solving $I$ as $B$ expands in solving $I$.
\end{theorem}
\begin{proof} Consider the two problem instances $I_1$ and $I_2$ in Figure~\ref{fig:2x}. In these instances $h_F(n)=h_B(n)=0$ for all $n$, $s$ is the start and $g$ is the goal. Assume $A$ is given either one of these instances as input. Since $A$ is DXBB and cannot initially distinguish $I_1$ from $I_2$, it must initially behave the same on both instances. Hence, on both instances, $A$ will initially either expand $s$ in the forward direction, expand $g$ in the backward direction, or expand both $s$ and $g$. 

\begin{figure}[b]
	\centerline{
		\includegraphics[width=6cm]{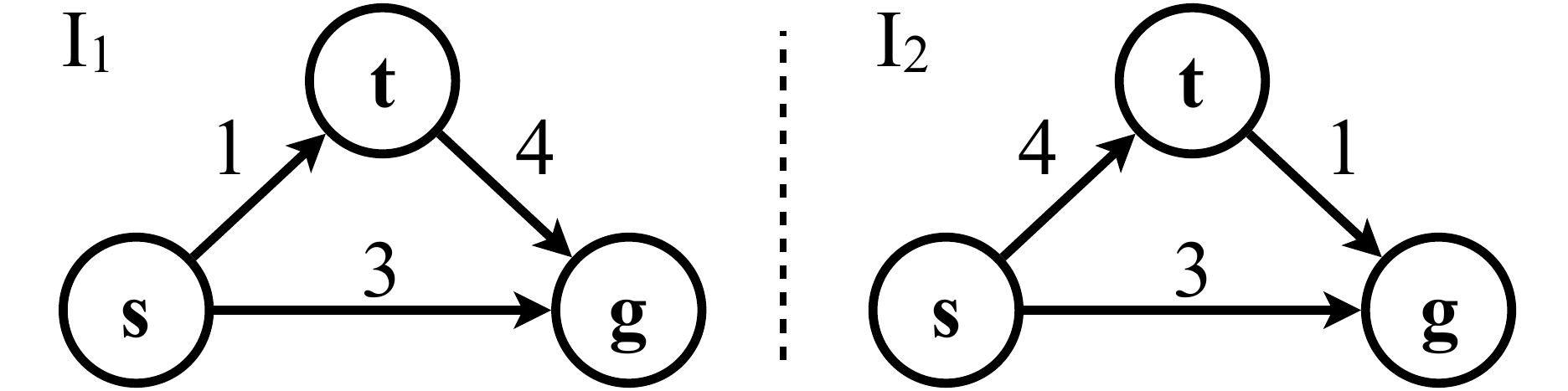}
	} \caption{Two problem instances with $C^*=3$ differing only in the costs of the edges $(s,t)$ and $(t,g)$.}
	\label{fig:2x}\vspace{-0.1cm}
\end{figure}
	\indent If $A$ first expands $s$ in the forward direction, consider $I=I_1$. On instance $I$, the algorithm $A$ has to expand a second state (either $g$ in the backward direction or $t$ in the forward direction) in order to be able to terminate with the optimal solution path $(s, g)$. By comparison, an algorithm $B$ that first expands $g$ in the backward direction will terminate with the optimal solution after just a single state expansion. Here we assume that $B$ terminates when there are no pairs satisfying the sufficient condition for node expansion.
	
	If $A$ first expands $g$ in the backward direction, one can argue completely symmetrically, with $I=I_2$ and $B$ being an algorithm that first expands $s$ in the forward direction.
	
 If $A$ begins by expanding both $s$ and $g$, as \NBA\ does, then on both these instances it will have expanded two states when only one expansion was required.
\end{proof}

\begin{figure}[t]
	\centering
	\includegraphics[width=6cm]{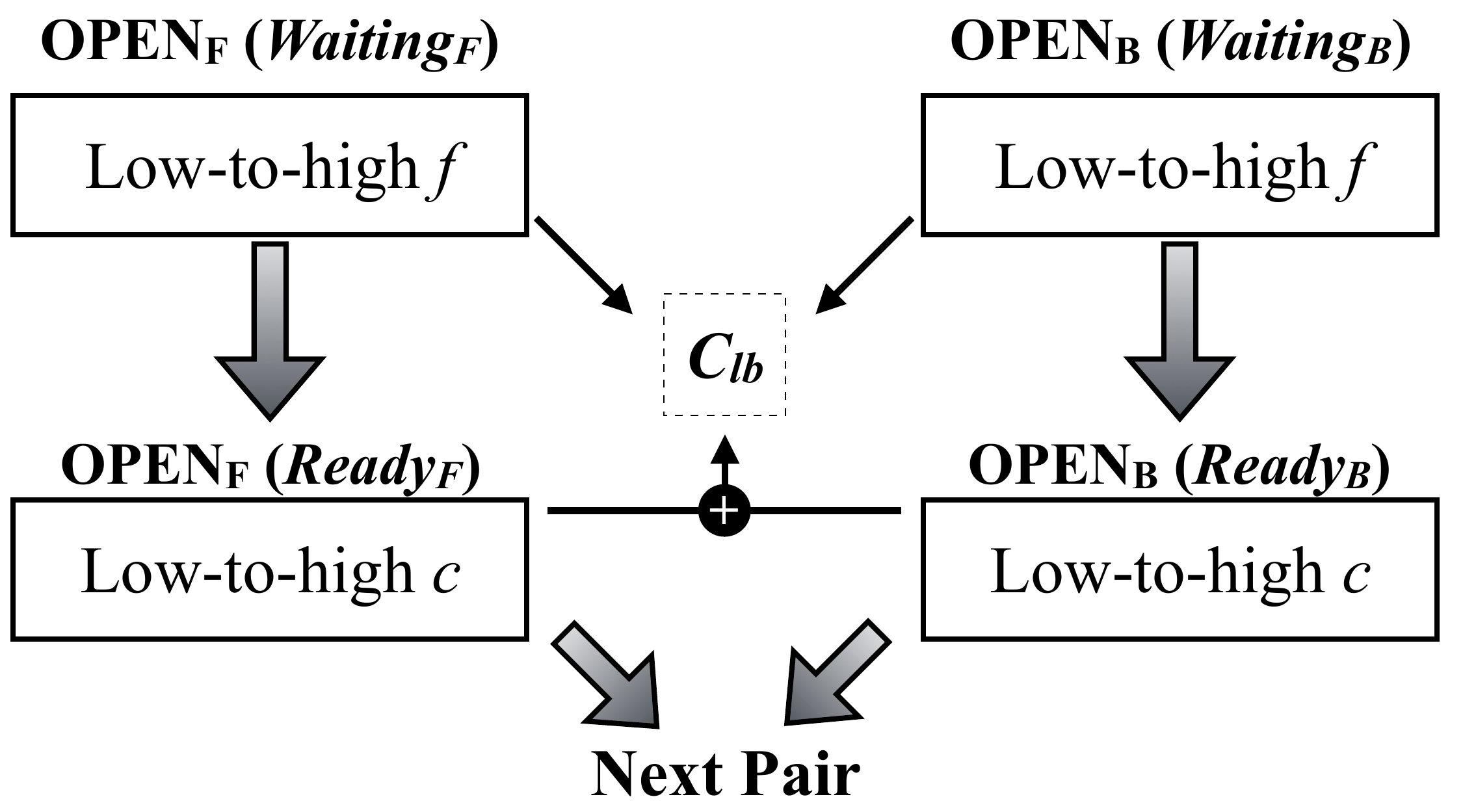}
	\caption{The open list data structure.}
	\label{fig:open-diagram}
\end{figure}

\begin{figure}[b!]
	\centering
	\includegraphics[width=8cm]{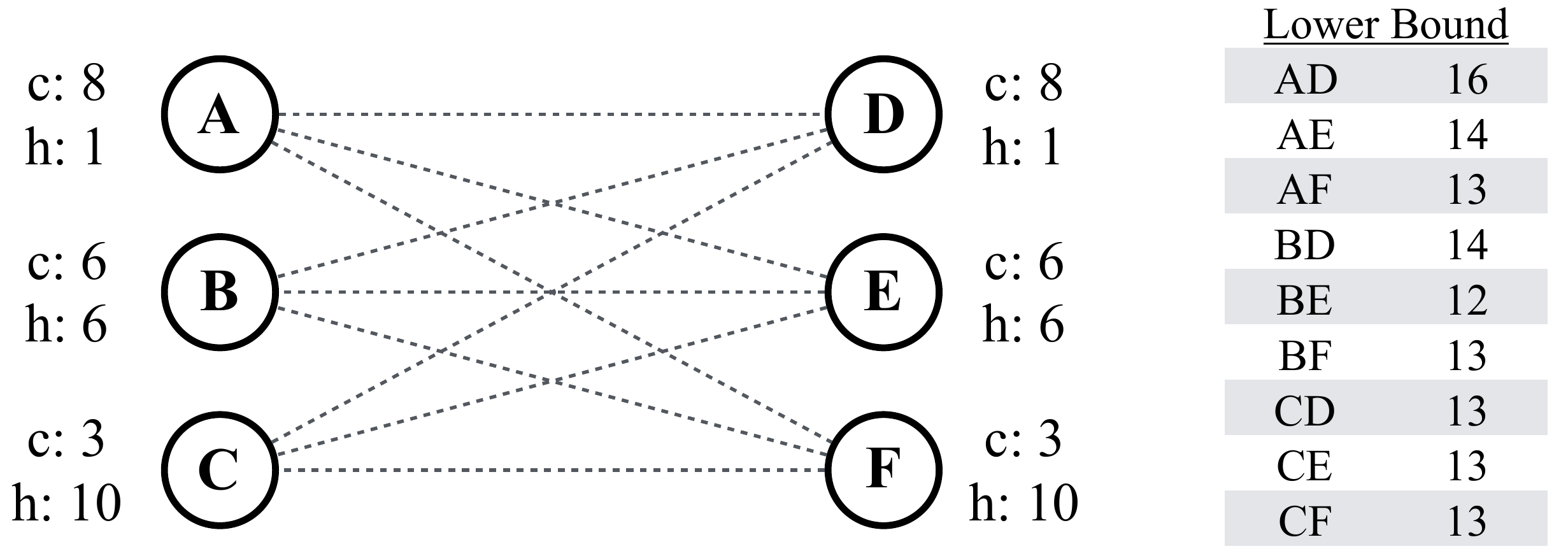}
	\caption{Sample state space and priorities of state pairs}
	\label{fig:open-example}
\end{figure}

\section{Efficient Selection of Paths for Expansion}\label{sec:datastructure}

Algorithm \ref{alg:nbs} assumes that \NBA\ can efficiently compute $lbmin$ and select the best path pair $(U, V)$ for expansion. In this section we provide a new \open\ list that can do this efficiently. The data structure works by maintaining a lower bound on $lbmin$, $C_{lb}$. \NBA\ initalizes $C_{lb}$ to $0$ prior to its first iteration and each time a new path pair is needed, $C_{lb}$ is raised until it reaches $lbmin$ and a new path pair is found.

This data structure is illustrated in Figure~\ref{fig:open-diagram}. The data structure is composed of two priority queues for each direction. The first priority queue is \waitF. It contains paths with $\fF \geq C_{lb}$ sorted from low to high \fF. The second priority queue is \readyF. It contains paths with $\fF \leq C_{lb}$ sorted from low to high \cF. Analogous queues are maintained in the backward direction. When paths are added to \open, they are first added to \waitF\ and \waitB. After processing, paths are removed from the front of \readyF\ and \readyB. The current value of $C_{lb}$ is the minimum of the $f$-costs at the front of \waitF\ and \waitB\ and the sum of the $c$-costs at the front of \readyF\ and \readyB.

\begin{algorithm}[t!]
	\caption{\NBA\ pseudocode for selecting the best pair from \open\ list. $C_{lb}$ is set to 0 when the search begins.}
	\relsize{-1}
	\begin{algorithmic}[1]
		\Procedure{PrepareBest}{}
			\While {min $f$ in \waitD\ $< C_{lb}$}\label{alg:ds-low}
				\State move best node from \waitD\ to \readyD\
			\EndWhile
			\While {true}
				\If {\readyD\ $\cup$ \waitD\ empty}\label{alg:ds-fail}
					\Return false
				\EndIf
				\If {\readyF$.c$ + \readyB$.c \leq C_{lb}$} \label{alg:ds-success}
					\Return true
				\EndIf
				\If {\waitD$.f \leq C_{lb}$ } \label{alg:ds-equalf}
					\State move best node from \waitD\ to \readyD\
				\Else
					\State $C_{lb}$ = min(\waitF$.f$, \waitB$.f$, \readyF$.c$+\readyB$.c$) \label{alg:ds-increase}
				\EndIf
			\EndWhile
			\EndProcedure
	\end{algorithmic} \label{alg:data_structure}
\end{algorithm}

Pseudocode for the data structure is in Algorithm~\ref{alg:data_structure}. Where forward or backwards queues are designated with a $D$, operations must be performed twice, once in each direction. We use the notation \readyF$.c$ to indicate the smallest $c$-cost on \ready\ in the forward direction. At the end of the procedure the paths on \ready$_F$ and \ready$_B$ with the smallest individual $c$-costs together form the pair to be expanded.

The procedure works as follows.
First, paths with $f$-cost lower than $C_{lb}$ must immediately be moved to \ready\ (line \ref{alg:ds-low}). If \readyD\ and \waitD\ are jointly empty in either direction, the procedure is halted and the search will terminate (line \ref{alg:ds-fail}). If the best paths in \ready\ have $\cF(U)+\cB(V) \leq C_{lb}$, the procedure completes; these paths will be expanded next (line \ref{alg:ds-success}).

If the \readyD\ queue is empty in either direction, any paths with $f=C_{lb}$ can be moved to \ready\ (line \ref{alg:ds-equalf}). While we could, in theory, move all such paths to \ready\ in one step, doing so incrementally allows us to break ties on equal $f$ towards higher $c$ first, which slightly improves performance. If there are no paths with $f \leq C_{lb}$ in \wait\ and in \ready\ with $\cF(U)+\cB(V) \leq C_{lb}$, then the $C_{lb}$ estimate is too low, and $C_{lb}$ must be increased (line \ref{alg:ds-increase}).

We illustrate this on an artificial example from Figure~\ref{fig:open-example}.\footnote{This example also illustrates why we cannot just sort by minimum \fF\ or \cF\ when performing expansions.}
To begin, $C_{lb} \leftarrow 0$ and we assume that $(A, B, C)$ are on \waitF\ and $(D, E, F)$ are on \waitB. First, $C_{lb}$ is set to 9 (line \ref{alg:ds-increase}). Then, $A$ and $D$ can be added to \ready\ because they have lowest $f$-cost, and $C_{lb} = 9$. However, $\cF(A) + \cF(D) = 16 > C_{lb} = 9$, so we cannot expand $A$ and $D$. Instead, we increase $C_{lb}$ to 12 and then add $B$ and $E$ to ready. Now, because $\cF(B) + \cF(E) = 12 \leq C_{lb}$ we can expand $B$ and $E$.

We can prove that the amortized runtime over a sequence of basic operations of our data structure (including insertion and removal operations) is $O(\log(n))$, where $n$ is the size of \wait\ $\cup$ \ready.

\begin{table*}[tb]
\caption{Average state expansions for unidirectional (A*) and bidirectional search across domains.}
\begin{center}
\small
\begin{tabular}{|l|l|l|c||r||r|r|r||r|}\hline
Domain & Instances & Heuristic & Strength & A* & BS* & MMe & \NBA\ & MM$_0$ \\ \hline
Grids & DAO & Octile & $+$ & {\bf 9,646} & 11,501 & 13,013 & 12,085 & 17,634 \\
Grids & Mazes & Octile & $-$ &  64,002  &  42,164  &  51,074  & {\bf 34,474}  &  51,075 \\ \hline

4 Peg TOH & 50 random & 12+2 PDB & $++$ &  1,437,644 &  {\bf 1,106,189}  & 1,741,480 & 1,420,554 & 12,644,722  \\
4 Peg TOH & 50 random & 10+4 PDB & $-$ & 19,340,099 &  8,679,443  & 11,499,867 & {\bf 6,283,143} & 12,644,722 \\ \hline

16 Pancake & 50 random & GAP & $+++$ & {\bf 125} & 339 & 283 & 335 & {\em unsolvable}\\
16 Pancake & 50 random & GAP-2 & $-$ &  1,254,082 &  947,545 &  {\bf 578,283} &   625,900  & {\em unsolvable}\\
16 Pancake & 50 random & GAP-3 & $--$ &  {\em unsolvable} &  29,040,138 & 7,100,998 &   {\bf 6,682,497}  & {\em unsolvable}\\ \hline
15 puzzle & $[$Korf,1985$]$ & MD & $+$ &  15,549,689 & {\bf 12,001,024} & 13,162,312 & 12,851,889 & {\em unsolvable} \\ \hline
\end{tabular}
\end{center}
\label{results}
\end{table*}%

\begin{table}[tb]
	\caption{Average running time and expansions per second for unidirectional (A*) and bidirectional search across domains.}
	\begin{center}
		\small
		\begin{tabular}{|l|c||r||r|r|r|}\hline
			\multicolumn{6}{|c|}{{\bf Average Running Time} (in seconds)} \\ \hline
			Domain & $h$ & A* & BS* & MMe & \NBA  \\ \hline
			DAO &  Octile & 0.005 & 0.006 & 0.015 & 0.007 \\
			Mazes & Octile  & 0.035 & 0.022 & 0.060 & 0.019  \\ \hline
			TOH4 & 12+2 &  3.23 & 2.44 & 4.17 & 3.54   \\
			TOH4 & 10+4 & 52.08 & 23.06 & 30.64 & 16.60  \\ \hline
			Pancake & GAP  & 0.00 & 0.00 & 0.00 & 0.00   \\
			Pancake & GAP-2  & 14.16 & 4.91 & 5.25 & 5.23   \\ 
			Pancake & GAP-3  & N/A & 212.33 & 72.13 & 77.17  \\ \hline
			15 puzzle  & MD & 47.68 &  29.59 & 41.38 & 37.67  \\ \hline
			\multicolumn{6}{|c|}{{\bf Expansion Rate} ($\times10^3$ nodes per seconds)} \\ \hline
			DAO &  Octile &  1,896 & 1,912  &  851 & 1,662  \\
			Mazes & Octile  & 2,225 & 2,366 &  848 & 2,290 \\ \hline
			TOH4 & 12+2 &  444 & 453 & 418 & 401    \\
			TOH4 & 10+4 & 371 & 376 & 375 & 379  \\ \hline
			Pancake & GAP  &  156 &  564  &  564  &  153    \\
			Pancake & GAP-2  &  89 &  193  &  109  &  120    \\ 
			Pancake & GAP-3  & N/A & 137 & 98 & 87  \\ \hline
			15 puzzle  & MD &  326 &   406 &  318  &  338  \\ \hline
		\end{tabular}
	\end{center}
	\label{time_results}
\end{table}%

\section{Experimental Results}

\begin{figure}[tb]
	\centering
	\includegraphics[width=6cm]{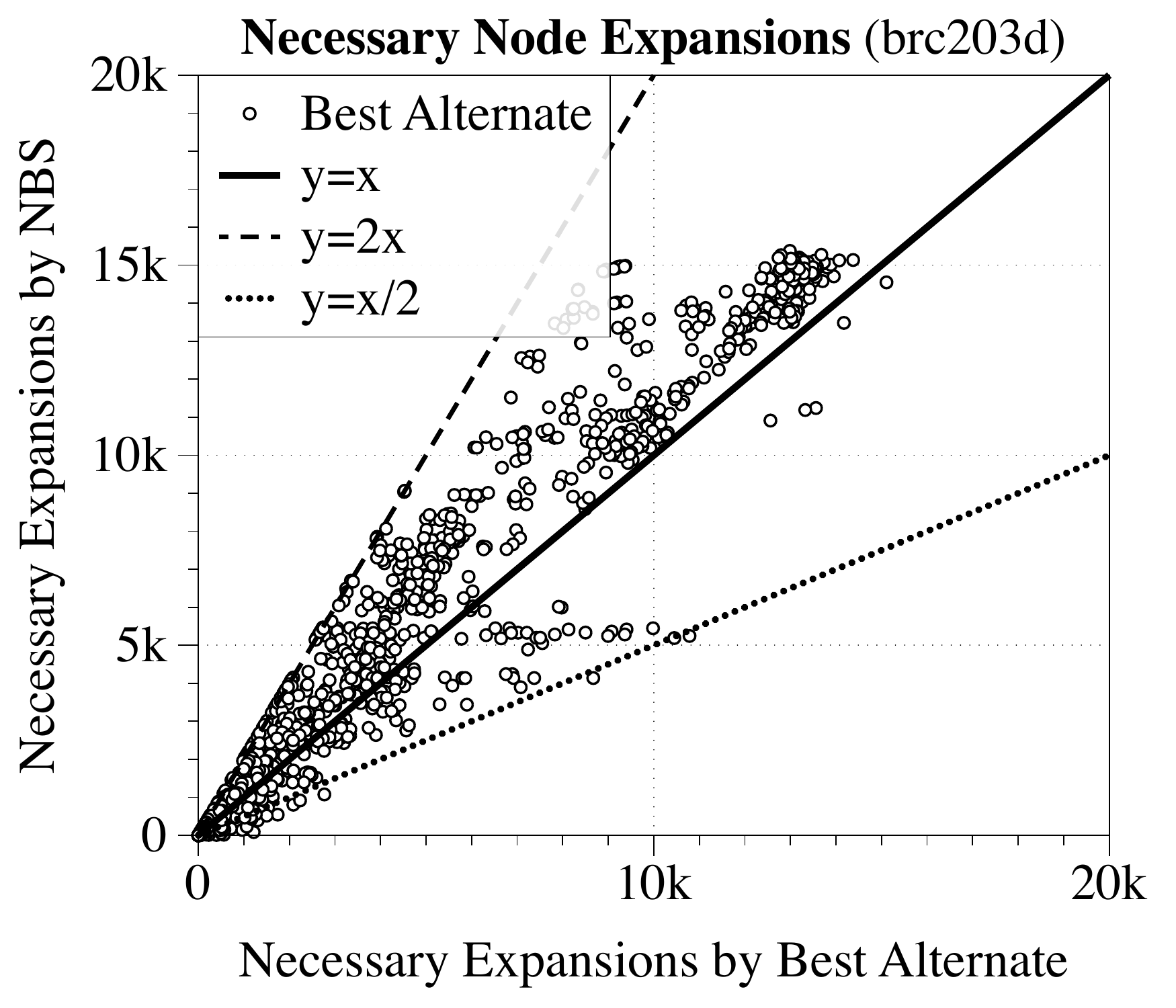}
	\caption{A comparison between necessary expansions by \NBA\ (y-axis) and the minimum of the expansions with $f<\Cstar$ by MMe, BS* and A* (x-axis) on each problem instance.}
	\label{fig:scatter}
\end{figure}

There are two primary purposes of the experimental results. First, we want to validate that our implementation matches the theoretical claims about \NBA. Second, we want to compare the overall performance of \NBA\ to existing algorithms. This comparison includes total node expansions, necessary expansions, and time. We compare A*, BS*\footnote{BS* is not an admissible algorithm (it will not optimally solve problems with an inconsistent heuristic) so the theory in this paper does not fully apply to BS*.}, MMe (a variant of MM) \cite{sharon2016mme}, \NBA, and MM$_0$ (bidirectional brute-force search).
 We also looked at IDA*, but it was not competitive in most domains due to cycles.

In Table \ref{results}
we present results on problems from four different domains, including grid-based pathfinding problems~\cite{sturtevant2012benchmarks} (`brc' maps from Dragon Age: Origins (DAO)), random 4-peg Tower of Hanoi (TOH) problems, random pancake puzzles, and the standard 15 puzzle instances~\cite{korf85}. The canonical goal state is used for all puzzle problems.\footnote{On the 15-puzzle and TOH it is more efficient to search backwards because of the lower branching factor, but we search forward to the standard goal states.} On each of these domains we use standard heuristics of different strength. The octile, GAP~\cite{Helmert10}, and Manhattan Distance heuristics can be easily computed at runtime for any two states. The additive pattern databases used for Towers of Hanoi are computed for each instance. We selected the size of problems to ensure that as many algorithms as possible could solve them in RAM.

In grid maps we varied the difficulty of the problem by changing the map type/topology. In TOH and the pancake puzzle we varied the strength of the heuristic. The GAP-$k$ heuristic is the same as the GAP heuristic, except that gaps involving the first $k$ pancakes are not added to the heuristic. The approximate heuristic strength on a problem is indicated by a $+$ or $-$. The general trend is that with very strong heuristics, A* has the best performance. As heuristics get weaker, or the problems get harder, the bidirectional approaches improve relative to A*. \NBA\ is never far from the best algorithm, and on some problems, such as TOH, it has significantly better performance than all previous approaches.
Runtime and node expansions/second are found in Table~\ref{time_results}. \NBA\ is 30\% slower than A* on the DAO problems, but competitive on other problems. \NBA\ is slower than BS*, but this is often compensated for by performing fewer node expansions. 
%

In Table~\ref{ratio_results} we look at the percentage of total nodes on closed compared to total expansions with $f = \Cstar$ by each algorithm in each domain. For the majority of domain and heuristic combinations there are very few expansions with $f = \Cstar$. The exception is the pancake puzzle with the GAP heuristic. On random instances this heuristic is often perfect, so all states expanded have $f = \Cstar$. This is why \NBA\ does more than twice the number of expansions as A* on these problems---these expansions are not accounted for in our theoretical analysis. BS* puts nodes on closed that it does not expand, which is why it has a negative percentage.

\begin{table}[tb]
	\caption{Percent of expansions with ($f$-cost $=\Cstar$) for each algorithm/domains.}
	\begin{center}
		\small
		\begin{tabular}{|l|c||r||r|r|r|}\hline
			Domain & Heuristic & A* & BS* & MMe & \NBA  \\ \hline
			DAO &  Octile & 1.3\% & 0.6\% & 0.7\% & 1.2\%  \\
			Mazes & Octile  & 0.0\% & 0.0\% & 0.0\% & 0.0\%  \\ \hline
			TOH4 & 12+2 PDB & 0.0\% & 0.0\% & 0.0\% & 0.0\%  \\
			TOH4 & 10+4 PDB & 0.0\% & 0.0\% & 0.0\% & 0.0\%  \\ \hline
			Pancake & GAP  & 60.7\% & 81.6\% & 76.2\% & 81.0\%  \\
			Pancake & GAP-2  & 0.2\% & 1.6\% & 6.0\% & 0.0\%  \\ 
			Pancake & GAP-3  & N/A & -0.6\% & 5.7\% & 0.0\%  \\ \hline
			15 puzzle  & MD & 5.5\% &  0.5\%  &  0.6\% &  0.3\%    \\ \hline
		\end{tabular}
	\end{center}
	\label{ratio_results}
\end{table}%


Figure \ref{fig:scatter} shows a scatter plot of necessary node expansions on 1171 instances from the {\em brc203d} grid map where \NBA\ has slightly worse performance than A* and BS*, but better performance than MM.
Each point represents one problem instance, and plots \NBA\ necessary expansions against the {\em minimum} of the expansions with $f < \Cstar$ by A*, BS* and MMe.
The $y=2x$ line represents the theoretical maximum expansion ratio (Theorem~\ref{cor:bound}), which is never crossed.
\NBA\ often does much better than this theoretical worst case, and on approximately 30 instances is more than 2x better than all alternate algorithms, as they do
not have similar 2x expansion bounds.

\section{Conclusion}

This paper presents the first front-to-end heuristic search algorithm that is near-optimal in 
necessary node expansions. It thus addresses questions dating back to Pohl's work on
the applicability of bidirectional heuristic search~\cite{pohl71}. When the problems are hard or the
heuristic is not strong, \NBA\ provides a compelling alternative to A*.

%

%
%
%

\section{Acknowledgements}
Financial support for this research was in part provided by Canada's Natural Sciences and
Engineering Research Council (NSERC). This material is based upon work supported by the
National Science Foundation under Grant No. 1551406.

\bibliographystyle{named}
\bibliography{ref}

\end{document}